\documentclass[sigconf,nonacm]{acmart}

\setcopyright{none}
\acmConference{}{}{}
\acmBooktitle{}
\acmDOI{}
\acmISBN{}
\acmYear{}

\AtBeginDocument{%
  
}

\usepackage[ruled,vlined,linesnumbered]{algorithm2e}
\usepackage{placeins}

\begin{document}

\title{Active Localization of Unstable Systems with Coarse Information}

\author{Ege Yuceel, Daniel Liberzon, Sayan Mitra}
\affiliation{%
  \institution{University of Illinois Urbana-Champaign}
  \department{Coordinated Science Laboratory}
  \city{Urbana}
  \state{IL}
  \country{USA}
}
\email{{eyceel2,liberzon,mitras}@illinois.edu}

\begin{abstract}

We study localization and control for unstable systems under coarse, single-bit sensing. Motivated by understanding the fundamental limitations imposed by such minimal feedback, we identify sufficient conditions under which the initial state can be recovered despite instability and extremely sparse measurements. Building on these conditions, we develop an active localization algorithm that integrates a set-based estimator with a control strategy derived from Voronoi partitions, which provably estimates the initial state while ensuring the agent remains in informative regions. Under the derived conditions, the proposed approach guarantees exponential contraction of the initial-state uncertainty, and the result is further supported by numerical experiments. These findings can offer theoretical insight into localization in robotics, where sensing is often limited to coarse abstractions such as keyframes, segmentations, or line-based features.

\end{abstract}

\maketitle

\theoremstyle{remark}
\newtheorem{remark}{Remark}
\newtheorem*{informalthm}{Theorem (Preview of Theorem~6.3)}

\ccsdesc[300]{Theory of computation~Timed and hybrid models}
\ccsdesc[500]{Computer systems organization~Robotics}

\keywords{Active localization, set-based estimation, control, robotics}

\section{Introduction}
\label{sec:intro}
Precise localization with  limited sensing is a fundamental challenge in many engineering systems. Contemporary  robotic solutions---Visual Inertial Odometry (VIO)~\cite{mourikis2007multi,bloesch2015robust} and Simultaneous Localization and Mapping (SLAM)~\cite{dissanayake2001solution,thrun2002probabilistic,engel2017direct}---typically rely on rich, continuous streams from cameras, inertial measurements units (IMU), gyroscopes, and often GPS.  In contrast, we study the extreme sensing regime: localization from discrete, intermittent, and sparse measurements, motivated by lightweight, low-power agents and minimalistic robotics~\cite{sabatini2013low,lin2020set,liberzon2025indistinguishability,sakcak2024mathematical}.

%
Prior theoretical  results have characterized which states are observable using combinations of visual and intertial measurements~\cite{bryson2008observability,jones2011visual,kelly2009visual}.
For example, it has been shown that the velocity, the absolute scale, the gravity vector in the local frame are observable ~\cite{martinelli2014closed}.
Prior work in quantized and hybrid control has  examined  stabilizability and observability~\cite{liberzon2003hybrid,sur1998state}. 
However, theoretical understanding of the limits on accuracy of localization with  coarse, binary sensing had remained unexplored until recently.

Concretely, we study the problem of localizing  an agent with known discrete-time dynamics \(x_{k+1} = f(x_k, u_k)\). The true  state $x_k$ is unobserved. At each step $k$, the agent receives a binary observation 
$y_k = h(x_k, m)$, indicating whether $x_k$ lies near a fixed landmark at location $m$. Although  $x_0$ and $m$ are unknown, they are known to be in the prior sets  $X_0$ and $M$, respectively.  
The goal is to design an algorithm that iteratively computes $\langle u_k, \hat{X}_{0,k}\rangle = g(y_0, \ldots, y_k, u_0, \ldots, u_{k-1})$, where $u_k$ is the control input and $\hat{X}_{0,k}$ is an estimate for $x_0$ computed after step $k$. For accurate estimation\footnote{Accurately estimating the initial state $x_0$ implies that the current state $x_k$ can also be estimated by propagating the dynamics and the control forward.},  we would like $\hat{X}_{0,k}$ to converge to  $x_0$. 
The challenge arises from the coarseness of binary observations and the instability of the dynamics: the former provides limited state information, while the latter inflates uncertainty, reducing the likelihood of receiving informative measurements.

In a recent paper~\cite{liberzon2025indistinguishability}, the authors characterize lower bounds on localization accuracy in terms of \emph{indistinguishable states} and relate them to Kalman’s observability decomposition. A pair of location–landmark configurations \((x_1, m_1)\) and \((x_2, m_2)\) is said to be \emph{indistinguishable} if they yield identical measurements for every possible sequence of control inputs. It is shown that no localization algorithm can shrink the initial state estimate \(\hat{X}_{0,k}\) beyond the limit imposed by the indistinguishable set (without additional knowledge of \(m\)). For example, for volume-preserving dynamics (such as \(f(x_k, u_k) = x_k + u_k\)), the paper shows that the indistinguishable sets are not singletons, and therefore, perfect localization is impossible. Instead, the indistinguishable sets correspond to diagonals; that is, a configuration \((x_1, m_1)\) is indistinguishable from \((x_0, m)\) if \(x_1 - m_1 = x_0 - m\). It is worth noting that finding  algorithms like $g$ that compute the control sequences to achieve the  localization, up to the possible level of accuracy, remained an open problem. 

In this paper, we consider linear unstable dynamics and seek algorithms that solve the localization problem optimally, in the sense of achieving the smallest possible uncertainty set. Given the unstable dynamics, at first glance, one might expect larger indistinguishable sets and worse localization performance. Somewhat counter-intuitively, we show that the indistinguishable sets collapse to singletons, implying that the initial state can, in principle, be recovered exactly. This is the first contribution of this work: We show that under unstable dynamics, any incorrect initial state guess \(x_1 \in \hat{X}_0\) with \(x_1 \neq x_0\) produces a trajectory that eventually diverges
from the true trajectory, and therefore, if the true trajectory can be made to produce positive measurements, then $x_1$ is becomes distinguishable.
We show that sets of such incorrect guesses can be algorithmically  eliminated from \(\hat{X}_{0,k}\) based on the observations. The elimination process does not require knowledge of \(m\), provided that a sufficient number of positive measurements are observed over time. 

While instability aids estimation, it also makes it difficult to keep the state within the sensing region, which is necessary to obtain  positive measurements. A carefully designed control strategy is therefore required when the signal is lost. The second contribution of this work is the design of such a strategy, which \textit{provably} returns the system to the sensing region. To this end, we decompose the estimate of the sensing region into subregions from which the state may escape. Each subregion is paired with a control law that guarantees robust recovery for all escape points within it. This idea of creating the recovery control by decomposing state uncertainty utilizes geometric properties of  \textit{spherical Voronoi partitions}~\cite{na2002voronoi} and may be of independent interest.

The paper is organized as follows. Section~\ref{sec:prelim} introduces the required preliminaries, including notation and spherical Voronoi partitions, and Section~\ref{sec:problem_def}  formally states the problem. Section~\ref{sec:estimate} shows how persistent positive measurements can be used to refine uncertainty in the agent’s initial state and the landmark location. Section~\ref{sec:recovery} presents a Voronoi-based recovery strategy to handle signal loss and establishes conditions under which the system provably re-enters the sensing region. Section~\ref{sec:combine} unifies these results in an algorithm and shows that the agent visits the sensing ball infinitely often, and this leads to estimation of the initial state. Section~\ref{sec:experiments} shows experimental results. 

Overall, we establish lower bounds on the levels of instability and initial uncertainty $X_0$ under which the initial state and landmark can be localized from single-bit proximity measurements. Building on this analysis, we propose \hyperref[alg:ACT-LOC]{\textsc{ActiveLocalize}}, an algorithm that guarantees exponential contraction of the initial state estimate $\hat{X}_{0,k}$. Finally, we validate the proposed approach through randomized experiments that consistently support the theoretical results.



\section{Preliminaries} \label{sec:prelim}
\paragraph{Notation}
The ellipsoid $\mathcal{E}(\mu, P) = \{ x \in \mathbb{R}^d \mid (x - \mu)^{\top} P (x - \mu) \le 1 \}$ is centered at $\mu \in \mathbb{R}^d$ with symmetric and positive definite shape matrix $P \in \mathbb{R}^{d \times d}$. The diameter of a set $A \subset \mathbb{R}^n$ is defined by $\text{diam}(A) = \sup\{ \|x - y\|_2 \mid x, y \in A\}$, where $\|\cdot\|_2$ denotes the Euclidean norm. $\mathcal{B}(x, r) = \mathcal{E}(x, \frac{1}{r^2} I)$ denotes the closed Euclidean ball of radius $r$ centered at $x$. $\mathbb{S}^{n-1} := \{ x \in \mathbb{R}^n : \|x\| = 1 \}$ denotes the unit shell. The convex hull of a set $A$, denoted by $\text{conv}(A)$, is the set of all convex combinations of points in $A$, that is, $\text{conv}(A) = \{ \sum_{i=1}^{n} \lambda_i x_i \mid x_i \in A, \lambda_i \ge 0, \sum_{i=1}^{n} \lambda_i = 1 \}$. The Minkowski sum of two sets $X, Y \subset \mathbb{R}^n$ is the pointwise addition of their elements, that is, $X \oplus Y = \{ x + y \mid x \in X,\, y \in Y \}$ .
\subsection{Spherical Voronoi Partition}
Spherical Voronoi partition provides formal mechanism for quantifying angular separation on the unit sphere and will serve as a key geometric tool in the analysis throughout the paper. In this section, we introduce the concept of Spherical Voronoi Partition (SVP) that partitions the $n$-dimensional unit shell into regions based on a selected set of directions. We then show that, for a certain type of SVP, any two points falling in the same region are guaranteed to be closely aligned, and demonstrate a special geometric property of such SVPs that will be useful throughout the paper. The rest of the section presents a method for computing such SVPs.

\begin{definition} \label{def:SVP}
Let $p_i \in \mathbb{R}^n$ for $i = 1, \dots, N$, and consider a ball $\mathcal{B}(c,\rho)$ with $c \in \mathbb{R}^n$ and $\rho > 0$.  
The spherical Voronoi partition (SVP) of $\scalebox{1.2}{$\partial$}\mathcal{B}(c,\rho)$ is defined by the cells
\begin{equation}
\label{eq:svp_def}
R_i
= \bigl\{\,x \in \mathbb{R}^n : \|x - c\| = \rho,\ \langle x - c, p_i\rangle \ge \langle x - c, p_j\rangle \ \forall j \neq i \bigr\}.
\end{equation}
We define the interior of $R_i$ as $\text{int}(R_i) := \text{conv}(R_i \cup \{c\})$. Lastly, for any SVP $R_{1:N}$, the maximum alignment $\eta$ is defined as
\begin{equation}
\label{eq:alpha_eta}
\eta = \max_{i \neq j} \langle p_i, p_j\rangle.
\end{equation}
\end{definition}


\begin{remark}
For a given maximum alignment $\eta$, the maximum number of unit vectors $p_i$ that one can place on $\mathbb{S}^{n-1}$ varies monotonically with $\eta$. As $\eta$ increases, the angular‐separation constraint relaxes and $N$ grows; as $\eta$ decreases, the vectors must be more widely spaced and $N$ shrinks. A classical case is $\eta = \tfrac{1}{2}$, which corresponds to the kissing‐number problem, requiring the placement of the maximal number of non‐overlapping unit spheres tangent to a central one
\cite{boyvalenkov2015surveykissingnumbers}. In two dimensions, the optimum is achieved by six basis vectors whose convex hull form a regular hexagon. Other highly symmetric configurations include the twelve basis vectors in three dimensions forming a regular icosahedron, which has $\eta=1/\sqrt5<\tfrac12$ \cite{zboroczky2017stability}. However, in higher dimensions symmetric basis vectors are rare, and an optimization problem has to be solved to determine $N$.
\end{remark}

The above are well-established results. We now turn to a specific construction of SVPs which provides a lower bound on two vectors in an SVP. This bound will play an important role later in the analysis.

\begin{definition}\label{def:alphaSVP}
Let $\alpha \in [0,1)$ be a constant.  
An $\alpha$-SVP is the partition $R_{1:N}$ generated by basis vectors $p_{1:N} \subset \mathbb{S}^{n-1}$ such that the caps
\[
C(p_i,\alpha) = \{\, w \in \mathbb{S}^{n-1} : \langle p_i, w \rangle > \alpha \,\}
\]
cover the unit shell, i.e.\ $\mathbb{S}^{n-1} \subseteq \bigcup_{i=1}^N C(p_i,\alpha)$, with the smallest $N$.
\end{definition}

\begin{figure}[h]
  \centering
  \begin{minipage}{0.34\columnwidth}
    \centering
    \includegraphics[width=\linewidth]{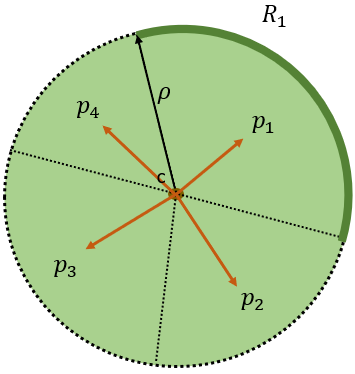}
  \end{minipage}\hspace{0.1\columnwidth}
  \begin{minipage}{0.43\columnwidth}
    \centering
    \includegraphics[width=\linewidth]{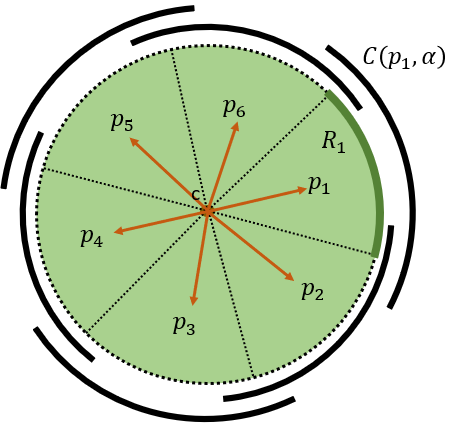}
  \end{minipage}
  \caption{(a) An SVP spanned by four basis vectors (orange arrows), with the first partition corresponding to $p_1$ highlighted along the boundary in dark green; 
  (b) an $\alpha$-SVP of $\scalebox{1.2}{$\partial$}\mathcal{B}(c,\rho)$, where the caps appear as black arcs and one cap, $C(p_1,\alpha)$, with its partition $R_1$, is highlighted.}
  \label{fig:svp_comparison}
\end{figure}

\begin{proposition}\label{prop:SVP_exist}
For any $\alpha \in [0,1)$, an $\alpha$-SVP always exists.
\end{proposition}
\begin{proof}
See Appendix~\ref{appendix:SVP_existproof}.
\end{proof}
\begin{proposition}\label{prop:SVP_bound}
Let $\alpha \in [0,1)$ and let $R_{1:N}$ be $\alpha$-SVP of $\scalebox{1.2}{$\partial$}\mathcal{B}(c,1)$ for any $c$ with basis vectors $p_{1:N}$.  
For any $i$ and $w_1, w_2 \in R_i$, we have
\[
\langle w_1, w_2 \rangle \;\ge\; 2\alpha^2 - 1.
\]
\end{proposition}
\begin{proof}
See Appendix~\ref{appendix:SVP_boundproof}.
\end{proof}
Although the purpose of spherical Voronoi partitioning is not completely clear at this point, we will show in Section \ref{sec:recovery} that the recovery control that brings the agent back to $\mathcal{B}(m,r)$ will be a function of the basis vectors of a certain $\alpha$-SVP. Now, having shown that any two points in the same $\alpha$-SVP are well aligned, 
we now use this property to derive a containment result for $\alpha$-SVPs within the sensing ball. 

\begin{proposition} \label{prop:SVP_ball}
Let \(x\in \mathcal{B}(m,r)\) and let \(R_{1:N}\) be a \(\tfrac{\sqrt{3}}{2}\)-SVP of $\scalebox{1.2}{$\partial$}\mathcal{B}(x,r)$. Then there exists \(i\in\{1,\dots,N\}\) such that
\[
int(R_i)\subseteq \mathcal{B}(m,r).
\]
\end{proposition}
\begin{proof}
See Appendix~\ref{appendix:SVP_ballproof}.
\end{proof}
\paragraph{Constructing $\alpha$-SVPs}

In practice, we need to find basis vectors 
$p_{1:N} \subset \mathbb{S}^{n-1}$ 
such that the corresponding regions $R_{1:N}$ form an $\alpha$-SVP. 
Each basis vector lies on the unit sphere in $\mathbb{R}^n$, and our goal is to arrange them 
to be as mutually separated as possible. This can be formulated as the following 
optimization problem:
\[
\label{eq:svp-direct}
\begin{aligned}
\min_{\{p_i\}} \quad & \gamma, \\
\text{s.t.}\quad & \langle p_i, p_j \rangle \le \gamma, \quad \forall i \ne j, \\
& \|p_i\|_2 = 1, \quad \forall i = 1,\dots,N.
\end{aligned}
\]
Here, the scalar $\gamma$ bounds the maximum pairwise inner product 
between distinct vectors, and minimizing $\gamma$ pushes the basis vectors to be as far apart as possible on the unit sphere. We denote this optimization process by \textsc{VecOpt}. Efficient procedures for computing \textsc{VecOpt} have already been developed in the literature \cite{Homsup_2018}, and a detailed treatment is beyond the scope of this paper. For an $\alpha$-SVP, we need the minimum $N$. Therefore, we start with a small $N$ and iteratively solve this optimization while incrementing $N$ until the caps $\{C(p_i,\alpha)\}_{i=1}^{i=N}$ cover the unit sphere, where $\alpha$ is fixed. This procedure is given in \hyperref[alg:FindSVP]{Algorithm 1}.

\begin{algorithm}[h]
\caption{FindSVP} \label{alg:FindSVP}
\KwIn{$\alpha \in [0,1)$, $N$, $n$}
\KwOut{$p_{1:N} \subset \mathbb{S}^{n-1}$}

\Repeat{all $m_i \ge \alpha$}{
   $p_{1:N} \gets \textsc{VecOpt}(N, n)$\;
  $R_{1:N} \gets \textsc{Voronoi}(p_{1:N})$\;

  \ForEach{$R_{1:N}$}{
    $m_i \gets \min_{w\in R_i}\langle p_i,w\rangle$\;
    \If{$m_i < \alpha$}{
      $N \gets N+1$
      \textbf{break} 
    }
  }
}
\Return $p_{1:N}$\;
\end{algorithm}

\begin{remark}
Finding the minimum inner product in Line 5 of the  \hyperref[alg:FindSVP]{FindSVP} algorithm is not expensive since each \(R_i\) forms a convex region on the sphere and \(\langle p_i, w\rangle\) is linear in \(w\). The minimum is therefore attained at the vertices of \(R_i\), which correspond to intersections of bisector boundaries and can be found analytically in low dimensions. If \(\langle p_i, w\rangle \ge \alpha\) for all \(i\), the bound \(\langle w_1, w_2\rangle \ge 2\alpha^2 - 1\) follows for all \(w_1, w_2 \in R_i\) as in Proposition~\ref{prop:SVP_bound}.
\end{remark}


\section{Problem Formulation} \label{sec:problem_def}
We consider an agent moving in Euclidean space attempting to localize itself with the aid of a landmark. The real position of the agent at time $k$, is $x_k \in \mathbb{R}^n$ and the position of the stationary landmark is $m \in \mathbb{R}^n$. The agent does not know its real position $x_k$ nor the position of the landmark $m$. Instead, it knows about two sets $X_0, M \subset \mathbb{R}^n$ such the  landmark $m \in M$ and its initial position $x_0 \in X_0$. Both \(X_0\) and \(M\) are convex polyhedra. The dynamics of the agent are described by the linear system
\begin{equation}
x_{k+1} = A\,x_k + Bu_k,
\label{eq:system_dynamics}
\end{equation}
where \(x_k \in \mathbb{R}^n\) denotes the system state at time step \(k\), \(u_k\) is the control input, \((A, B)\) is a controllable pair, and the system matrix \(A\) is unstable, i.e., \(|\lambda_i(A)| > 1\) for all \(i\). At each time step \(k\), the agent receives a single-bit measurement
\begin{equation}
y_k=
\begin{cases}
1, & \text{if } \| x_k - m \| \leq r, \\
0, & \text{otherwise},
\end{cases}
\label{c}
\end{equation}
with \(r > 0\) being the detection radius associated with the landmark. Assume that $x_0 \in \mathcal{B}(m, r)$ so that $y_0 = 1$. For estimation purposes, we denote the state and landmark estimates internally computed by the agent as \(\hat{X}_k[t]\) and \(\hat{M}[t]\), respectively. Here, \(\hat{X}_k[t]\) represents the estimate\footnote{The notation $\hat{X}_{0,k}$ used in the introduction for the informal problem statement is replaced by $\hat{X}_0[k]$ hereafter, to improve clarity.} of \(x_k\) computed at time \(t\); since the landmark is stationary, we drop the subscripts for the estimates of $m$, i.e.,  \(\hat{M}[t]\) represents the estimate of \(m\) at time \(t\). Naturally, at the beginning we have $\hat{X}_0[0] = X_0$ and $\hat{M}[0] =M$. The objective is to design an algorithm that, given \(X_0\), \(M\), \(r\), \(u\) and the dynamics \eqref{eq:system_dynamics}, jointly synthesizes a control policy and an estimator such that
\[
\lim_{k\to\infty}\text{diam}\bigl(\hat X_0[k]\bigr)=0.
\]
Since the result involves several technical steps and intermediate constructions, 
we first present an informal statement to highlight the main idea.

\begin{informalthm}
If the landmark uncertainty satisfies a suitable condition linked to the system’s instability
(see~\eqref{eq:general-cond}), then 
\hyperref[alg:ACT-LOC]{\textsc{ActiveLocalize}} guarantees exponential contraction of 
the initial state estimate:
\[
\mathrm{diam}(\hat X_0[k]) \le C\, a^{k},
\qquad C>0,\; 0<a<1.
\]
\end{informalthm}

We denote by \(\textsc{Reach}(\cdot)\) the operator that computes the set of all states reachable from the initial set \(X_0\) under the dynamics~\eqref{eq:system_dynamics} and input sequence \(u_{0:k-1}\). The operator is formally defined as  
\[
\textsc{Reach}(A,B,u_{0:k-1},X_0)
= \left\{\, A^k x_0 + \sum_{i=0}^{k-1} A^{k-1-i} B u_i \;\middle|\; x_0 \in X_0 \,\right\}.
\]
Although the exact computation of the reachable set may be computationally expensive, tight over-approximations can be obtained by propogating finite number of vertices whose convex hull tightly over-approximates \(X_0\).
\section{State Estimation under Persistent Measurements} \label{sec:estimate}

First, we show that if we can choose a control sequence that somehow maintains $y_k=1$ for a stretch of steps, then the uncertainty estimate $\hat{X}_k[k]$ can be kept within some bound using our estimation method. Unfortunately, it may not be possible to keep $y_k = 1$ because the system is unstable and the exact state $x_k$ is unknown. To address this, we introduce a second control mechanism that steers the state back into the detection ball $\mathcal{B}(m,r)$. The signal can then be recovered if $\hat{X}_k[k]$ is sufficiently small. Together, we combine these arguments to guarantee recovery of $\mathcal{B}(m,r)$ after any signal loss, and show this leads to an exponential decay of the initial-state uncertainty.

We now give an overview of \hyperref[alg:EST]{\textsc{Estimate}}. Fix $x_0$, a sequence of controls $u_{0:k-1}$, and let $L_k$ be the subset of time indices $\{0,\dots,k\}$ that have positive measurements, including $k$. The algorithm takes two sets $X_0, M$, the filtered index set $L_k$, and the control sequence $u_{0:k-1}$, and produces the updated sets $X_0'$, $X_k'$, and $M'$. These sets will later correspond to $\hat{X}_0[k]$, $\hat{X}_k[k]$, and $\hat{M}[k]$, respectively. Algorithm initializes a time-iteration variable $w=0$ to track $L_k$, and sets $\bar X_0[w]=X_0$. In \hyperref[EST:line3]{Line 3} and \hyperref[EST:line4]{4}, for each $j \in L_k \setminus \{k\}$, an ellipsoid center $\mu_{kj}$ and shape matrix $P_{kj}$ are computed. In \hyperref[EST:line5]{Line 5}, $X_0$ is tightened by intersecting it with $\mathcal{E}(\mu_{kj}, P_{kj})$. In \hyperref[EST:line8]{Line 8}, the tightened set $X_0'$ is propagated through the linear dynamics using $\textsc{Reach}(\cdot)$ operator and intersected with the bloated set $M$ to obtain $X_k'$. In \hyperref[EST:line9]{Line 9}, $M'$ is obtained by intersecting $M$ with the reachable sets at times $j \in L_k$, each similarly bloated. Finally, it returns the updated sets. Also, \hyperref[alg:subroutines]{Algorithm 2} lists two computational subroutines used by \hyperref[alg:EST]{\textsc{Estimate}} to compute the ellipsoids required for shrinking $X_0$. 
The specific choice of these subroutines will become clear in Lemma~\ref{lem:init_state_est}.

\begin{algorithm}
\caption{Subroutines for \hyperref[alg:EST]{\textsc{Estimate}}}
\label{alg:subroutines}

\SetKwProg{Fn}{Function}{}{}
\Fn{\textsc{ComputeMu}($k,j,u,A,B$)}{
  $\mu_{kj} \gets -(A^k - A^j)^{-1}\!\left(
    \sum_{m=0}^{k-1} A^{k-1-m}Bu_m
    - \sum_{m=0}^{j-1} A^{j-1-m}Bu_m
  \right)$\;
  \Return{$\mu_{kj}$}\;
}

\Fn{\textsc{ComputeEllipsoidMatrix}($k,j,A,r$)}{
  $P_{kj} \gets \dfrac{(A^k - A^j)^\top (A^k - A^j)}{4r^2}$\;
  \Return{$P_{kj}$}\;
}

\end{algorithm}

\begin{algorithm}
\caption{\hyperref[alg:EST]{\textsc{Estimate}}(\(X_0,\, M,\, L_k,\, u_{0:k-1}\))}
\label{alg:EST}

$w \gets 0$; 
$\bar{X}_0[w] \gets X_0$;  \label{EST:line1}

\For{$j \in L_k \setminus \{k\}$}{ \label{EST:line2}
  $\mu_{kj} \gets \textsc{ComputeMu}(k,j,u,A,B)$\; \label{EST:line3}
  $P_{kj} \gets \textsc{ComputeEllipsoidMatrix}(k,j,A,r)$\; \label{EST:line4}
  $\bar{X}_0[w+1] \gets \bar{X}_0[w] \cap \mathcal{E}(\mu_{kj},P_{kj})$\; \label{EST:line5}
  $w \gets w+1$\; \label{EST:line6}
}

$X_0' \gets \bar{X}_0[w]$\; \label{EST:line7}
$X_k' \gets \textsc{Reach}(A,B,u_{0:k-1},X_0') \cap (M \oplus \mathcal{B}(0,r))$\; \label{EST:line8}
$M' \gets \bigcap_{j \in L_k} \left( \textsc{Reach}\!\left(A, B, u_{0:j-1}, X_0'\right) \oplus \mathcal{B}(0, r) \right) \cap M$\; \label{EST:line9}

\Return{$(X_0',X_k',M')$}\; \label{EST:line10}
\end{algorithm}

Before we start with the formal statements, we outline their intent. In our setting, the estimates \((\hat{X}_0[k],\, \hat{X}_k[k],\, \hat{M}[k])\) are updated jointly from their previous values by \hyperref[alg:EST]{\textsc{Estimate}}. Lemma~\ref{lem:init_state_est} shows that the update of the initial-state set \(\hat{X}_0[k]\) in \hyperref[alg:EST]{\textsc{Estimate}} is an over-approximation of \(x_0\). Corollary~\ref{cor:landmark_est} builds on Lemma~\ref{lem:init_state_est} and shows that the landmark estimate \(\hat{M}[k]\) is an over-approximation of \(m\), and Corollary~\ref{cor:cur_state_est} proves that the current-state estimate \(\hat{X}_k[k]\) is an over-approximation of \(x_k\).

\begin{lemma} \label{lem:init_state_est}
Consider any $k^* > 0$ and a control sequence $u_{0:{k^*-1}}$. 
For all $k \le k^*$, inductively define
\[
(\hat{X}_0[k], \hat{X}_k[k], \hat{M}[k]) 
= \hyperref[alg:EST]{\textsc{Estimate}}\big(\hat{X}_0[k-1],\, \hat{M}[k-1],\, L_k,\, u_{0:{k-1}}\big)
\]
whenever $y_k = 1$, and 
$(\hat{X}_0[k],\, \hat{M}[k]) 
= (\hat{X}_0[k-1],\, \hat{M}[k-1])$ otherwise,
with initial conditions $\hat{X}_0[0] = X_0$ and $\hat{M}[0] = M$. 
Then
\[
x_0 \in \hat{X}_0[k^*].
\]
\end{lemma}

\begin{proof}
We proceed by induction on $k \le k^*$. 
The claim holds at $k = 0$, since $x_0 \in X_0 = \hat{X}_0[0]$. 
Assume $x_0 \in \hat{X}_0[k-1]$ for some $k > 0$. 
If $y_k = 0$, then by definition $\hat{X}_0[k] = \hat{X}_0[k-1]$, so $x_0 \in \hat{X}_0[k]$. 
Otherwise, if $y_k = 1$, fix $L_k$, an arbitrary $x_0 \in \hat{X}_0[k-1]$, and the control sequence $u_{0:k-1}$. 
Since $k$ is the last element of $L_k$ and $L_k$ is monotonically increasing, for every $j \in L_k$ with $j \neq k$ we have $j < k$. 
If we show that $x_0 \in \mathcal{E}(\mu_{kj}, P_{kj})$ for all $j \in L_k$ with $j < k$, this implies 
$x_0 \in \bigcap_{j < k,\, j \in L_k} \mathcal{E}(\mu_{kj}, P_{kj}) 
\implies x_0 \in \hat{X}_0[k]$, 
where the intersection corresponds to the estimator update in \hyperref[EST:line5]{Line 5}. 
Hence, by induction, $x_0 \in \hat{X}_0[k]$ for all $k \le k^*$. 

The remainder of the proof establishes $x_0\in\mathcal E(\mu_{kj},P_{kj})$ for each $k>j\in L_k$. We show this by recognizing that for any control sequence \(u\), if there exist time indices \(k\) and \(j\) with \(y_k=y_j=1\), then the difference of corresponding states \(x_k\) and \(x_j\), lies within a ball that can be determined from the available information, even though the true measurement ball \(\mathcal{B}(m,r)\) is unknown. By recording the controls applied up to time \(k\) and propagating this determinable ball backward through the linear dynamics, we obtain an ellipsoid guaranteed to contain the initial state \(x_0\). Start with an immediate observation about \(x_k\) and \(x_j\): Using \eqref{c}, we can write
\begin{equation}
  \|x_k-x_j\|\le2r. \label{eq:inside_cond}
\end{equation}
Setting \(A_{kj}=A^k - A^j\) and unrolling the dynamics we have
\[
x_k-x_j
= A_{kj}\,x_0
+ \sum_{m=0}^{k-1}A^{\,k-1-m}Bu_m
- \sum_{m=0}^{j-1}A^{\,j-1-m}Bu_m.
\]
Note that $u_m$ up to $m=k-1$ and $m=j-1$ always exist since $k>j$ and we are given $u_{0:k-1}$. Define
\[
d_{kj}
=\sum_{m=0}^{k-1}A^{\,k-1-m}Bu_m
\;-\;\sum_{m=0}^{j-1}A^{\,j-1-m}Bu_m, \mbox{so \ that}
\]
\[
x_k-x_j
= A_{kj}\,x_0 + d_{kj} \implies \|A_{kj}\,x_0 + d_{kj}\|\le2r.
\]
Squaring and expanding yields
\begin{align*}
x_0^\top A_{kj}^\top A_{kj}\,x_0
+2\,d_{kj}^\top A_{kj}\,x_0
+d_{kj}^\top d_{kj}
\le (2r)^2.
\end{align*}
Factoring out \({A_{kj}}^\top A_{kj}\) and completing the square, we obtain
\[
\Bigl(x_0+A_{kj}^{-1}d_{kj}\Bigr)^\top
\frac{A_{kj}^\top A_{kj}}{(2r)^2}
\Bigl(x_0+A_{kj}^{-1}d_{kj}\Bigr)
\le1.
\]
Identifying
$\mu_{kj}=-A_{kj}^{-1}d_{kj}$ and  
$P_{kj}=\frac{A_{kj}^\top A_{kj}}{(2r)^2}$, this inequality is equivalent to
$\bigl(x_0-\mu_{kj}\bigr)^\top P_{kj}\,\bigl(x_0-\mu_{kj}\bigr)\le1$,
where $\mu_{kj}$ and $P_{kj}$ are as defined in \hyperref[alg:subroutines]{Algorithm 2}. Then \(x_0\in\mathcal E(\mu_{kj},P_{kj})\) and the proof is complete.
\end{proof}

\begin{remark}
The definition of $\mu_{kj}$ and $P_{kj}$ is valid since $A_{kj}=A^k-A^j$ is invertible. 
Because $A$ is diagonalizable (or has a Jordan form) with eigenvalues $|\lambda|>1$, 
each block of $A_{kj}$ has entries $\lambda^k-\lambda^j\neq0$ for $k>j\ge0$, 
implying $\det(A_{kj})\neq0$.
\end{remark}

\begin{remark}
While we consider the single landmark case in this paper, extension to multiple landmarks is possible. With multiple landmarks and unknown signal--landmark association, \hyperref[alg:EST]{\textsc{Estimate}} must be modified. Suppose there are \(n\) landmarks with common detection radius \(r\), and let
\(L_k=\{k_1<\cdots<k_q\}\) denote the indices of \(q\) positive measurements observed up to time \(k\). To guarantee consistency without data association, we require \(q\ge n+1\). For each pair \((k_a,k_b)\) with \(1\le a<b\le q\), form the ellipsoid
\(\mathcal E_{ab}:=\mathcal E(\mu_{k_b k_a},P_{k_b k_a})\) as in Algorithm~2, and eliminate a candidate \(x_0\) only if it violates all pairwise constraints. Thus, line 5 should be changed in the following way:
\[
\hat X_0[w+1]
\;\gets \;
\hat X_0[w]
\ \cap\
\Bigl(\ \bigcup_{1\le a<b\le q}\mathcal E_{ab}\ \Bigr),
\]
where the union is taken over the \(\binom{q}{2}\) measurement pairs.
\end{remark}

\begin{corollary} \label{cor:landmark_est}
With the same definitions of $\hat{X}_0[\cdot]$ and $\hat{M}[\cdot]$, we have
\[
m \in \hat{M}[k^*].
\]
\end{corollary}

\begin{proof}
Since $m \in M = \hat{M}[0]$, we proceed by induction on $k \le k^*$.  
If $y_k = 0$, then $\hat{M}[k] = \hat{M}[k-1]$, so $m \in \hat{M}[k]$.  
If $y_k = 1$, we show that the update performed in \hyperref[alg:EST]{\textsc{Estimate}} also preserves membership, i.e., $m \in \hat{M}[k]$. Consider any $\hat{X}_0[k]$ computed in \hyperref[EST:line7]{Line 7}
 of \hyperref[alg:EST]{\textsc{Estimate}}. 
By Lemma~\ref{lem:init_state_est}, for any $k\leq k^*$, we have
$x_0 \;\in\; \hat{X}_0[k]$.
Since the system is linear, propagating the tightened set $\hat{X}_0[k]$ through the dynamics preserves set membership. Hence, as computed in \hyperref[EST:line8]{Line 8}
 of 
\hyperref[alg:EST]{\textsc{Estimate}},
\begin{equation} \label{eq:linprop}
x_k \;\in\; \textsc{Reach}\!\left(A, B, u_{0:k-1}, \hat{X}_0[k]\right).
\end{equation}

Now, by the definition of $L_k$, for each $j \in L_k$ we have $y_j = 1$, 
which implies that the measurement point $m$ lies within a ball of radius $r$ 
centered at the corresponding state $x_j$, i.e.,
$m \in x_j \oplus \mathcal{B}(0,r)$. 
Using~\eqref{eq:linprop} and the definition of the Minkowski sum, this yields
\[
m \;\in\;
\textsc{Reach}\!\left(A, B,u_{0:j-1}, \hat{X}_0[k]\right)
\;\oplus\;
\mathcal{B}(0,r)
\quad \forall j \in L_k.
\]
Since this condition must hold for all indices in $L_k$, it follows that
\begin{align} \label{mupdate}
m \;\in\;&
\bigcap_{j \in L_k}
\left(
\textsc{Reach}\!\left(A, B, u_{0:j-1}, \hat{X}_0[k]\right)
\;\oplus\;
\mathcal{B}(0,r)
\right)
\cap \hat{M}[k-1] \notag\\
&=\; \hat{M}[k].
\end{align}

where~\eqref{mupdate} corresponds to the update performed in \hyperref[EST:line9]{Line 9}
 of 
\hyperref[alg:EST]{\textsc{Estimate}}. The proof is complete by induction.
\end{proof}

\begin{corollary} \label{cor:cur_state_est}
Let $y_{k^*} = 1$. With the same definitions of $\hat{X}_0[\cdot]$, $\hat{M}[\cdot]$, and $\hat{X}_k[\cdot]$, we have
\[
x_{k^*} \in \hat{X}_{k^*}[k^*].
\]
\end{corollary}

\begin{proof}
Let $k^*$ be such that $y_{k^*} = 1$. 
From Corollary~\ref{cor:landmark_est}, we have $m \in \hat{M}[k^*]$. 
By the definition of $L_{k^*}$, this implies 
$m \in x_{k^*} \oplus \mathcal{B}(0,r)$, or equivalently, 
$x_{k^*} \in \hat{M}[k^*] \oplus \mathcal{B}(0,r)$. 
Combining this with~\eqref{eq:linprop} gives
\begin{equation} \label{muppdate}
x_{k^*} \in 
\textsc{Reach}\!\left(A, B, u_{0:k^*-1}, \hat{X}_0[k]\right)
\cap
\bigl(\hat{M}[k^*] \oplus \mathcal{B}(0,r)\bigr)
= \hat{X}_{k^*}[k^*],
\end{equation}
where~\eqref{muppdate} corresponds to the update performed in \hyperref[EST:line8]{Line 8}
 of 
\hyperref[alg:EST]{\textsc{Estimate}}.
\end{proof}

Intuitively, Corollary~\ref{cor:landmark_est} follows from the fact that the state $x_k$ 
always lies within the reachable set obtained by propagating $\hat{X}_0[k]$ through the 
linear dynamics. When $y_k = 1$, bloating this reachable set with $\mathcal{B}(0,r)$ 
captures the landmark positions $m \in \hat{M}[k]$ for which the measurement at time $k$ 
can be positive. We do this for all $i \in L_k$ with $y_i = 1$. Conversely, by similar reasoning, the true state $x_k$ must lie in both 
$\hat{M}[k] \oplus \mathcal{B}(0,r)$ and the reachable set at time $k$, 
which is precisely the statement of Corollary~\ref{cor:cur_state_est}. A conceptual illustration of the \hyperref[alg:EST]{\textsc{Estimate}} is shown in Figure~\ref{fig:state_landmark}.

\begin{figure}[htbp]
  \centering
  \includegraphics[width=0.5\textwidth]{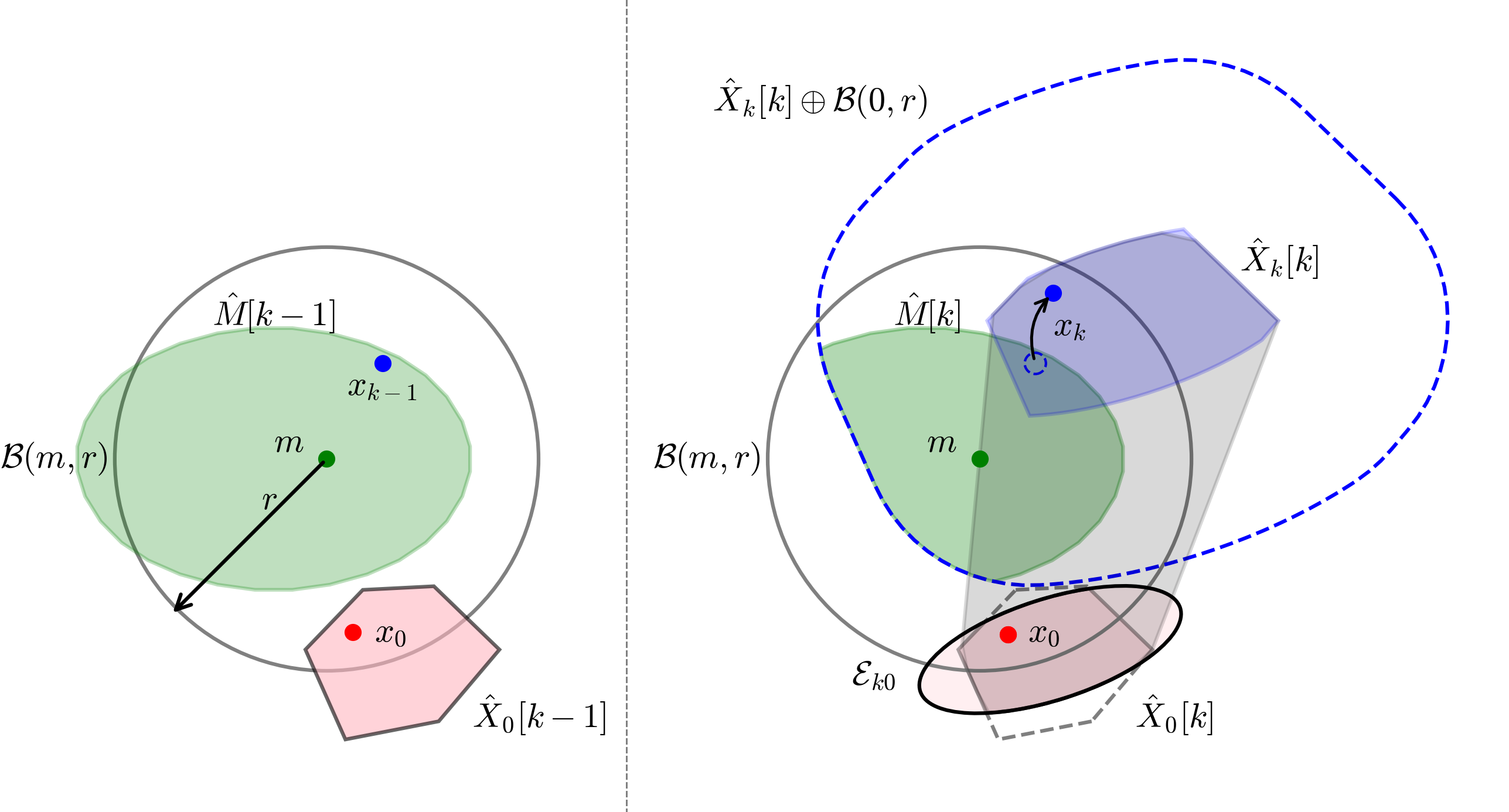}
  \caption{Initial-state estimate $\hat{X}_0[\cdot]$ and landmark estimate $\hat{M}[\cdot]$ at times $(k\!-\!1)$ (left) and $k$ (right). The state remains within the sensing region over $[0,k]$. For clarity, only one ellipsoid $\mathcal{E}(\mu_{k0},P_{k0})$—computed in \hyperref[alg:EST]{\textsc{Estimate}}~\hyperref[EST:line5]{Line~5}—is shown intersecting $\hat{X}_0[k-1]$. The gray tube illustrates $\textsc{Reach}(A,B,u_{0:k-1},\hat{X}_0[k])$. Since the intersection in~\eqref{muppdate} leaves this reachable set unchanged, it yields the updated current state estimate. Meanwhile, the same reachable set—only one shown for clarity—is inflated by $\mathcal{B}(0,r)$ to update $\hat{M}$.
}
  \label{fig:state_landmark}
\end{figure}

\section{Recovery Control Synthesis} \label{sec:recovery}
So far we have shown that as long as \(y=1\), we can refine the estimates for initial state $\hat{X}_0[k]$, current state $\hat{X}_k[k]$, and landmark $\hat{M}[k]$, and we have introduced spherical Voronoi partition and some of their properties. The instability of the system makes it difficult to keep the agent inside \(\mathcal{B}(m,r)\) permanently. Any control that stabilizes a point in \(\hat{X}_k[k]\) that is other than the true state \(x_k\) will cause \(x_k\) to diverge exponentially from \(\mathcal{B}(m,r)\) and the signal will be lost. This section shows that if \(y_k = 1\) and \(y_{k+1} = 0\) for some \(k \geq 0\), then under suitable conditions on the degree of instability and the size of the current state estimate, there exists a finite sequence of control inputs such that \(x_{k'} \in \mathcal{B}(m,r)\) for some \(k' > k\). 


From here until the end of this section only, we change the time indices. Let $k$ be the last time $y_k=1$ and $y_{k+1}=0$. For this $k$ we reset the time index and denote \(x_0 := x_k\), \( X_0 := \hat{X}_k[k] \) and $u_0 :=u_k$, without any loss of generality. Starting with this new time index, we construct the recovery signal beginning with the base case where $X_{0}$ is a ball and $B$ is the identity matrix, as defined in Definition~\ref{def:recovery_control_signal}. Lemma~\ref{lem:recoverability_base} shows that this control signal guarantees recovery of the system to the sensing region. We then relax these assumptions, extending the result in Corollary~\ref{cor:recoverability_region} to arbitrary compact initial sets and further generalizing it in Definition~\ref{def:generalized_RCS} and Corollary~\ref{cor:recoverability_region_control} to controllable systems.

\begin{definition}[Recovery Control Signal (RCS) ]\label{def:recovery_control_signal}
 Given \(x_0\in \mathcal{B}(c_0,r_0)\) for some $c_0$ and $r_0>0$ and basis vectors \(p_{1:N}\) defining $\alpha-$SVP for some $\alpha$, the \emph{recovery control signal} is recursively defined as
\begin{equation}
\label{eq:control_def}
u_i
\;=\;
\bigl(r - \|A^{\,i+1}-I\|\,r_0\bigr)\,p_i
\;-\;(A^{\,i+1}-I)c_0)
\;-\;\sum_{j=0}^{i-1}A^{\,i-j}u_j,
\
i=1,\dots,N.
\end{equation}
where $u_0$ is arbitrary.
\end{definition}

\begin{lemma}[Recoverability via Recovery Control Signal]\label{lem:recoverability_base}
For a $\tfrac{\sqrt{3}}{2}$-SVP with size $N$ and alignment $\eta$, set 
$D = \max_{1 \le k \le N} \|A^{k+1} - I\|$ and take $B = I$. If $x_0\in \mathcal{B}(m,r)\cap \mathcal{B}(c_0,R_0)$ and
\begin{equation}\label{eq:diam_assump}
r_0 \le \frac{r\,(1-\eta)}{D\,(3-\eta)},
\end{equation}
then for any RCS, $u_{1:N}$, $x_{i+1} \in \mathcal{B}(m,r)$ for some $i \in \{1,..,N\}$.
\end{lemma}

\begin{proof}
Let us fix $x_0 \in \mathcal{B}(m,r) \cap \mathcal{B}(c_0, r_0)$ and RCS $u_{1:N}$. We will first show $x_{i+1} \in int(R_i)$ for every $i$ where $R_i$ is the $i$th partition of $\scalebox{1.2}{$\partial$}\mathcal{B}(x_0,r)$. Then we use the result from Proposition~\ref{prop:SVP_ball} which tells $int(R_i) \subseteq \mathcal{B}(m,r)$ for some $i$ since $x_0 \in \mathcal{B}(m,r)$. Finally, since $x_{i+1} \in int(R_i)\ \forall i$ and $int(R_i) \subseteq \mathcal{B}(m,r)$ for some $i$, we will conclude $x_{i+1} \in \mathcal{B}(m,r)$ for some $i$.

Before proceeding, we outline what follows. We will start by successively rewriting $x_{i+1}$ in the form $x_{i+1} = x_0 + \beta_i$, where $\beta_i$ is decomposed as $\beta_i = \Delta + e\,p_i$ for explicitly defined $\Delta$ and $e$. The remainder of the proof will then establish that $\|\beta_i\| \le r \ \forall i$ and that $\langle \beta_i, p_i\rangle \ge \langle \beta_i, p_j\rangle\ \forall i \neq j$. These two properties together imply $x_{i+1} \in int(R_i) \ \forall i$. Starting from \eqref{eq:system_dynamics}, we have
\[
x_{i+1}
= A^{\,i+1}x_0 + \sum_{j=0}^{i}A^{\,i-j}u_j.
\]
The summation part can be separated into two parts: one up to $i-1$ and the $u_i$ term. Substituting \eqref{eq:control_def} for $u_i$, the terms up to $i-1$ cancel, and we obtain
\begin{equation}
\label{eq:grouped}
A^{\,i+1}x_0
+ \Bigl(r - \|A^{\,i+1}-I\|\,r_0\Bigr)p_i
- (A^{\,i+1}-I)c_0.
\end{equation}
Now, we can set $x_{i+1}$ to be \eqref{eq:grouped} and re-write the expression as $x_{i+1}=x_0+(\Delta+ep_i)$
where
\begin{equation} \label{e:ineq}
\Delta = (A^{\,i+1}-I)(x_0-c_0),
\quad
e = r - \|A^{\,i+1}-I\|\,r_0 
\end{equation}
Then, we bound $\Delta$ in the following way:
\begin{equation}\label{eq:deltabound}
\|\Delta\|\le \|A^{\,i+1}-I\|\,r_0 \le \frac{1-\eta}{3-\eta}\,r  
\end{equation}
where first inequality is from operator norm inequality and \(\|x_0-c_0\|\le r_0\), the second is from \eqref{eq:diam_assump}.  Now, define \(\beta_i = \Delta + e\,p_i\), so that $x_{i+1}=x_0+\beta_i$.  By the triangle inequality,
\begin{equation}
\label{eq:norm_beta}
\|\beta_i\|\le \|\Delta\| + e\|p_i\| \le \|A^{\,i+1}-I\|\,r_0 + (r - \|A^{\,i+1}-I\|\,r_0) = r.
\end{equation}
By Cauchy–Schwarz, $
\langle\Delta,p_i\rangle \ge -\|\Delta\|,
\
\langle\Delta,p_j\rangle \le \|\Delta\|.
$ Hence
\begin{equation}
\label{eq:inner1}
\langle \beta_i,p_i\rangle
= e + \langle \Delta,p_i\rangle
\ge e - \|\Delta\|,
\
\langle \beta_i,p_j\rangle
= e\,\langle p_i,p_j\rangle + \langle \Delta,p_j\rangle
\le e\,\eta + \|\Delta\|.
\end{equation}
where inequality in the second expression comes from \eqref{eq:alpha_eta} as $e\geq 0$ by \eqref{e:ineq}, \eqref{eq:diam_assump} and $0 \leq \frac{1-\eta}{3-\eta} \leq 1$ . Now using the definition of \(e\) at \eqref{e:ineq} with the second inequality in \eqref{eq:deltabound}, we write
\begin{equation}
\label{eq:e_lower}
e \ge r - \frac{1-\eta}{3-\eta}\,r = \frac{2r}{3-\eta}
\quad\Longrightarrow\quad
\frac{e(1-\eta)}{2} \ge \frac{(1-\eta)}{3-\eta}r.
\end{equation}
After the multiplication of both sides with $\frac{1-\eta}{2}$ at \eqref{eq:e_lower}, we can use \eqref{eq:deltabound} once again with \eqref{eq:e_lower} to obtain \begin{equation} \label{eq:delta_bound}
\|\Delta\| \le \frac{e(1-\eta)}{2}.
\end{equation}
Plugging \eqref{eq:delta_bound} to each inner product at \eqref{eq:inner1} yields the following:
\[
\langle \beta_i,p_i\rangle \ge e - \tfrac{1-\eta}{2}e = \tfrac{1+\eta}{2}e,
\quad
\langle \beta_i,p_j\rangle \le e\,\eta + \tfrac{1-\eta}{2}e = \tfrac{1+\eta}{2}e,
\]
so \(\langle \beta_i,p_i\rangle\ge\langle \beta_i,p_j\rangle \ \forall i \neq j\).  Together with $\|\beta_i\|\le r$ from \eqref{eq:norm_beta} and by the definition of SVP in \eqref{eq:svp_def}, we have
\begin{equation}
\label{eq:in_cell}
x_{i+1} = x_0 + \beta_i \;\in\; int(R_i) \quad \forall i.
\end{equation}
By Proposition ~\ref{prop:SVP_ball} there exists some \(i\) with \(int(R_i)\subseteq \mathcal{B}(m,r)\), and by \eqref{eq:in_cell} for that \(i\), \(x_{i+1}\in \mathcal{B}(m,r)\).
\end{proof}

\subsection{Generalizing the Recovery Control Signal}

We have shown that defining $X_0$ as a ball $\mathcal{B}(c_0, R_0)$ ensures signal recovery under the condition~\eqref{eq:diam_assump}, which relates the system’s instability rate to the initial uncertainty radius $r_0$ through the parameter $D$, a quantity that generally increases as the magnitude of $A$’s eigenvalues grows.
However, this description of \(X_0\) is restrictive and we can extend Lemma~\ref{lem:recoverability_base} with general compact $X_0$. To do so, we note that any bounded region $X_0$ can be over-approximated by an enclosing sphere whose radius is expressed in terms of \(\text{diam}(X_0)\) and \(n\), according to Jung’s theorem \cite{Jung1901}.

\begin{corollary}[Recoverability for General \(X_0\)]\label{cor:recoverability_region}
Let $X_0 \subset \mathbb{R}^n$, and let $c_0$ be the center of its smallest enclosing ball. Set the maximum norm $D$ and alignment $\eta$ as in Lemma~\ref{lem:recoverability_base}. If
\begin{equation}\label{cor_cond}
\text{diam}(X_0)\le\frac{r\,(1-\eta)}{D\,(3-\eta)}\sqrt{\frac{2(n+1)}{n}},
\end{equation}
then for any \(x_0\in X_0\cap \mathcal{B}(m,r)\), there exists some RCS, $u_{1:N}$, such that $x_{i+1}\in \mathcal{B}(m,r)$.
\end{corollary}
\begin{proof}
See Appendix~\ref{appendix:recoverability_regionproof}.
\end{proof}

The preceding result shows that finite--time recoverability holds for any bounded $X_0$ under the simplifying assumption $B = I$. This assumption corresponds to a fully actuated system and is used only for clarity. The same conclusion extends to under-actuated but controllable pairs $(A,B)$: if recovery is possible with single--step inputs when $B = I$, then it remains possible by grouping inputs into blocks of length equal to the controllability index. In Definition~\ref{def:generalized_RCS}, these blocks are indexed by $i$, as in \eqref{eq:control_def}, with the individual signals inside each block labeled by $j$.

\begin{definition}[Generalized RCS]\label{def:generalized_RCS}
Suppose $(A,B)$ is controllable in $\bar{n}$ steps. For each index \(i=1,\dots,N\) and sub-index \(j=1,\dots,\bar{n}\) set
\begin{equation}\label{eq:control_def_AB}
\begin{aligned}
u_{(i-1)\bar{n}+j}
&= \bigl(A^{\,\bar{n}-j}B\bigr)^{\top} W^{-1} \Bigl[
   (r - \|A^{\,\bar{n}i+1}-I\|\,R_{0})\,p_{i} \\
&\qquad\quad {} - (A^{\,\bar{n}i+1}-I)\,c_{0}
   - \sum_{k=0}^{(i-1)\bar{n}} A^{\,\bar{n}i-k}B\,u_k
   \Bigr].
\end{aligned}
\end{equation}

with arbitrary $u_0$ where $W$ is the $\bar{n}$‑step controllability Gramian defined  as $W=\sum_{k=0}^{\bar{n}-1}A^{k}B\,B^{T}(A^{k})^{T}$. Then the sequence $u_{1:N\bar{n}}$ is a generalized RCS.
\end{definition}

\begin{corollary}[Generalized Recoverability in Controllable Systems]\label{cor:recoverability_region_control}
Suppose $(A,B)$ is controllable and $X_0 \subset \mathbb{R}^n$ is bounded. 
Set alignment $\eta$ and size $N$ as in Lemma~\ref{lem:recoverability_base} and $\bar{D} := \max_{1 \le k \le N\bar{n}} \|A^{k+1} - I\|$. If
\begin{equation}\label{general_cond}
\text{diam}(X_0)\le\frac{r\,(1-\eta)}{\bar{D}\,(3-\eta)}\sqrt{\frac{2(n+1)}{n}},
\end{equation}
then for any \(x_0\in X_0\cap \mathcal{B}(m,r)\), there exists some generalized RCS, $u_{1:N\bar{n}}$, such that $x_{i+1}\in \mathcal{B}(m,r)$.
\end{corollary}
\begin{proof}
See Appendix~\ref{appendix:recoverability_region_controlproof}.
\end{proof}
\section{Putting It All Together} \label{sec:combine}
In the previous section, we showed that one-time recovery is possible under suitable conditions, and in Section \ref{sec:estimate}, we introduced a method to refine the initial-state uncertainty \( \hat{X}_0[k] \) whenever \( y_k = 1 \). We now show  that, for any increasing sequence of indices \( k_1 < k_2 < \ldots \) with \( y_{k_i} = 1 \), the size of the current state estimate set \( \hat{X}_{k_i}[k_i] \) remains upper bounded by a non-expanding value. Through Theorem~\ref{thm:main_thm}, we show that this sufficiently small bound ensures that, even under intermittent signal loss, the system repeatedly returns to $\mathcal{B}(m, r)$. This mechanism is realized in \hyperref[alg:ACT-LOC]{\textsc{ActiveLocalize}}, which unifies state estimation and recovery control synthesis.



\begin{algorithm}[t]
\caption{\textsc{ActiveLocalize}(\(X_0,\, M,\, r,\, \bar{n}\))}
\label{alg:ACT-LOC}
$k\!\gets\!0$;\ $\hat{X}_0[0]\!\gets\!X_0$;\ $\hat{M}[0]\!\gets\!M$;\ $L\!\gets\!\emptyset$;

\While{true}{
  observe $y_k$ at $x_k$\;
    \If{$y_k = 1$}{
        $L \gets L\cup\{k\}$;\ $u_{k}\!\gets$ arbitrary ctrl.;\
        $(\hat{X}_0[k],\hat{X}_k[k],\hat{M}[k]) \gets 
          \hyperref[alg:EST]{\textsc{Estimate}}(\hat{X}_0[k-1],\hat{M}[k-1],L,u_{1:k-1})$;\label{ACT-LOC:line5}\ 
        $x_{k+1}\!\gets\!A x_k + B u_k$;\ $k\!\gets\!k+1$ \tcp*{Update system} 
    }
  \Else{
    \For{$j=1$ \KwTo $N\bar{n}$}{
      $u_{k+j-1}\!\gets$ RCS[j] \tcp*{Compute from Def.~\ref{def:generalized_RCS}}\ 
      $x_{k+j}\!\gets\!A x_{k+j-1}+B u_{k+j-1}$;\ observe $y_{k+j}$\;
    \If{$y_{k+j}=1$}{
      $\hat{X}_0[k+j]\!\gets\!\hat{X}_0[k];\ \hat{M}[k+j]\!\gets\!\hat{M}[k]$  \label{ACT-LOC:line11}\;
      $k\!\gets\!k+j$;\ \textbf{break}\;
    }
    }
  }
}
\end{algorithm}

\begin{theorem} \label{thm:main_thm}
Suppose $(A,B)$ is controllable with controllability index $\bar{n}$. 
Let the size $N$, alignment $\eta$, and $\bar{D}$ be given, and let $M$ satisfy
\begin{equation}\label{eq:general-cond}
\frac{\mathrm{diam}(M)}{r} \;\le\; 
\frac{1-\eta}{\bar{D}\,(3-\eta)} 
\sqrt{\frac{2(n+1)}{n}} - 2.
\end{equation}
Then \hyperref[alg:ACT-LOC]{\textsc{ActiveLocalize}} achieves $y = 1$ infinitely often, with successive positive measurements separated by at most $N\bar{n}$ time steps.
\end{theorem}
\begin{proof}
The argument proceeds as follows: We will show applying \hyperref[alg:EST]{\textsc{Estimate}} at any time $k$ such that $y_k = 1$ (\hyperref[alg:ACT-LOC]{\textsc{ActiveLocalize}} \hyperref[ACT-LOC:line5]{Line 5}), gives
$
\text{diam}(\hat{X}_k[k]) \leq \text{diam}(M) + 2r.
$
We will use this result together with the condition ~\eqref{eq:general-cond} to show that condition~\eqref{general_cond} holds with 
$X_0 := \hat{X}_{k}[k]$, where $k$ is such that $y_{k} = 1$ and $y_{k+1} = 0$. Then we fix some generalized RCS, $u_{1:N\bar{n}}$, to be applied starting from time $k$. By Corollary ~\ref{cor:recoverability_region_control}, the generalized RCS ensures that, after signal loss at time $k$, there exists an index $k'$ satisfying $k < k' \leq k + N\bar{n}$ for which $y_{k'} = 1$. Then, we can use the same argument perpetually to arrive the desired result.

By simple calculations, the condition \eqref{eq:general-cond} implies  
\begin{equation} \label{modif}
\text{diam}(M) + 2r \;\leq\; 
\frac{r(1-\eta)}{\bar{D}(3-\eta)} \sqrt{\frac{2(n+1)}{n}}.
\end{equation}
Corollary~\ref{cor:landmark_est} shows \hyperref[alg:EST]{\textsc{Estimate}} produces the estimate $\hat{M}[k+1]$ that over-approximates $m$. This is updated in \hyperref[EST:line8]{Line 8} of  \hyperref[alg:EST]{\textsc{Estimate}} where the previous estimate, $\hat{M}[k]$, is intersected with a bloated version of several reachable sets. On the other hand, estimate remains the same if $y_k =0$ (\hyperref[ACT-LOC:line11]{Line 11} of \hyperref[alg:ACT-LOC]{\textsc{ActiveLocalize}}). Then for all $k \ge 0$, we have 
\begin{equation} \label{recur}
\text{diam}\bigl(\hat{M}[k+1]\bigr) \;\le\; 
\text{diam}\bigl(\hat{M}[k]\bigr) 
\;\le\; \text{diam}(M).
\end{equation}
Similarly Corollary~\ref{cor:cur_state_est} shows \hyperref[alg:EST]{\textsc{Estimate}} produces the estimate $\hat{X}_{k}[k]$ that over-approximates $x_k$. This is updated in \hyperref[EST:line9]{Line 9} of \hyperref[alg:EST]{\textsc{Estimate}} where the reachable set at time $k$, is intersected with the bloated version of the landmark estimate $\hat{M}[k]$. Thus, we have

\[
x_k \in \hat{X}_k[k]
\subseteq \hat{M}[k]\oplus B(r)\implies \text{diam}(\hat{X}_k[k]) \leq \text{diam}(\hat{M}[k])+2r
\]
by the Minkowski sum property. Thus, for any $k$ such that $y_k = 1$, we can combine this with ~\eqref{recur} to obtain

\[
\text{diam}\bigl(\hat{X}_k[k]\bigr)
\;\le\; \text{diam}(M) + 2r.
\]
Therefore, \eqref{modif} together with the above inequality yields,
\begin{equation} \label{invariance}
\text{diam}\bigl(\hat{X}_k[k]\bigr)
\;\le\; 
\frac{r(1-\eta)}{\bar{D}(3-\eta)}\sqrt{\frac{2(n+1)}{n}},
\end{equation}
whenever $y_k=1$. Now consider any arbitrary control $u_k$ such that $y_{k+1}=0$. Then for that $k$, \eqref{invariance} is the sufficient condition to Corollary~\ref{cor:recoverability_region_control} that ensures there exists a sequence of generalized RCS $u_{1:N\bar{n}}$ such that $y_{k'} = 1$ for some $k'$, with $k < k' \le k + N\bar{n}$. Then by invariance, $y = 1$ occurs infinitely often, with successive occurrences separated by at most $N\bar{n}$ time steps.
\end{proof}

We have shown that $x_k \in \mathcal{B}(m,r)$ for infinitely many times. We now prove that the initial‐state uncertainty decays exponentially. Lemma~\ref{lem:init_state_est} shows for any $k\geq0$, $x_0\in \hat{X}_0[k]$ where $\hat{X}_0[k]$ is the intersection of the prior, $X_0$, with all ellipsoids that is computed in \hyperref[alg:EST]{\textsc{Estimate}} up to time $k$. To prove exponential decay, we focus on one ellipsoid: take $j=0$ ($j\in L_k$ since $y_0=1$) and let $d_k = \max L_k$.  
Then $\hat{X}_0[k] \subseteq \mathcal{E}(P_{d_k0},\mu_{d_k0})$ .  
The key step is to analyze how the diameter of $\mathcal{E}(P_{d_k0},\mu_{d_k0})$ behaves as $d_k \to \infty$.  
Proposition~\ref{prop:matrixnorm} provides a necessary matrix inverse bound, and Theorem~\ref{thm:exp_decay} shows this diameter decays exponentially.
\begin{proposition} \label{prop:matrixnorm}
Define $
\lambda_{\min} \;=\;\min_{i}\lvert\lambda_i(A)\rvert > 1$.
Then there exist constants \(K>0\) and $0<c<1$ such that
\[
\bigl\lVert (A^k - I)^{-1}\bigr\rVert_2 \;\le\; K\,\lambda_{\min}^{-ck} \quad \forall k >0.
\]
\end{proposition}
\begin{proof}
See Appendix~\ref{appendix:matrixnormproof}.
\end{proof}
The above proposition will be used to provide a bound on the ellipsoid $\mathcal{E}(\mu_{k0}, P_{k0})$. 
The next objective is to extend this result by replacing the index $k$ with $d_k$. 
From Theorem~\ref{thm:main_thm}, which ensures that $y_k = 1$ occurs at least once every $N\bar{n}$ steps, it follows that $d_k \to \infty$ as $k \to \infty$ when \hyperref[alg:ACT-LOC]{\textsc{ActiveLocalize}} is used.
Consequently, this allows us to bound the ellipsoid $\mathcal{E}(\mu_{d_{k}0}, P_{d_{k}0})$ as well.

\begin{theorem} \label{thm:exp_decay}
Assume that condition \eqref{eq:general-cond} holds. 
Then, \hyperref[alg:ACT-LOC]{\textsc{ActiveLocalize}} achieves exponential contraction of the initial state estimate:
\begin{equation} \label{contraction}
\text{diam}(\hat{X}_0[k]) \le \min\{C\,a^k,\;\text{diam}(X_0)\}
\quad\forall k\ge0,
\end{equation}
for some finite $C>0$ and $0<a<1$ that depend on \eqref{eq:system_dynamics}.
\end{theorem}

\begin{proof}
Since, for any $k\geq0$, $d_k$ cannot be constant for \(N\bar{n}+1\) consecutive times, there always exists some index 
\(j \in [\,k-N\bar{n},\,k-1]\) such that \(d_{k} = j+1\). Then we have,
\begin{equation} \label{seq:bound}
d_k \ge k-N\bar{n} .
\end{equation}
By Lemma ~\ref{lem:init_state_est}, we  have  $x_0 \in \hat{X}_0[k] \subseteq \bigcap_{j \in L_k,\; j < d_k} \mathcal{E}(\mu_{d_k j}, P_{d_k j})$ 
which implies $\operatorname{diam}(\hat{X}_0[k]) \le \operatorname{diam}(\mathcal{E}(\mu_{d_k 0}, P_{d_k 0}))$ 
for all $k \ge 0$, where this ellipsoid is computed in $\hyperref[alg:EST]{\textsc{Estimate}}$ \hyperref[EST:line5]{Line 5} at the first iteration of the loop. Now we invoke Proposition ~\ref{prop:matrixnorm} and show that
\[
\begin{aligned}
\text{diam}\bigl(\hat{X}_0[k]\bigr)
&\leq \text{diam}\bigl(\mathcal{E}(\mu_{d_k0},P_{d_k0})\bigr) \\
&= \sup_{x,y\in\mathcal{E}(\mu_{d_k0},P_{d_k0})}\|x - y\|
= 2\sup_{\|(A^{d_k} - I)z\|\le 2r}\|z\| \\
&= 4r\,\bigl\|(A^{d_k} - I)^{-1}\bigr\|_2
\;\le\; 4rK\,\lambda_{\min}^{-c d_k}.
\end{aligned}
\]
where the first equality comes from the definition of diameter, the second comes from converting the supremum on diameter to supremum on the radius, and the third comes from the definition of operator norm with normalization and the last inequality is a direct consequence of Proposition ~\ref{prop:matrixnorm}. We also have $x_0 \in X_0$, which gives
\[
\begin{aligned}
&\text{diam}\bigl(\hat{X}_0[k]\bigr) 
\le \min\!\left\{4rK\,\lambda_{\min}^{-c d_k},\, \text{diam}(X_0)\right\} \\
&\le \min\!\left\{4rK\,\lambda_{\min}^{-c (k-N\bar{n})},\, \text{diam}(X_0)\right\}
   = \min\!\left\{Ca^k,\, \text{diam}(X_0)\right\}
\end{aligned}
\]
as \(d_k \geq k-N\bar{n}\) by \eqref{seq:bound}, where 
\(C := 4rK\,\lambda_{\min}(A)^{cN\bar{n}}\) and \(a := \lambda_{\min}(A)^{-c}\).  
Since \(\lambda_{\min}(A)^{c} > 1\), it follows that \(0 < a < 1\).
\end{proof}

\section{Numerical Examples} \label{sec:experiments}

We consider the discrete-time linear system in~\eqref{eq:system_dynamics} with randomly initialized $A, B \in \mathbb{R}^{2\times 2}$, and focus on two randomly obtained systems that are controllable and satisfy~\eqref{eq:general-cond}. \textbf{Setup~1} has minimum absolute eigenvalue $\lambda_{\min} = 1.014$ (plotted in red), and \textbf{Setup~2} has $\lambda_{\min} = 1.01$ (plotted in blue); in both cases we take $r = 2$. The initial state set $X_0$ is a hyperrectangle of side length $3.5$ centered at $[0.2,\,-0.2]$, and the landmark uncertainty set $M$ is a hyperrectangle of side length $1.0$ centered at $[0.5,\,0.5]$. We ran $40$ independent trials of \hyperref[alg:ACT-LOC]{\textsc{ActiveLocalize}}, each initialized with $x_0 \sim \mathrm{Unif}(X_0)$ and $m \sim \mathrm{Unif}(M)$, and simulated each trial for $500$ iterations. Because the closed-loop evolution is deterministic once initialized, at each iteration we report the minimum and maximum diameter across the $40$ trials rather than a standard deviation band.

\begin{figure}[h]
  \centering
  \includegraphics[width=0.9\columnwidth]{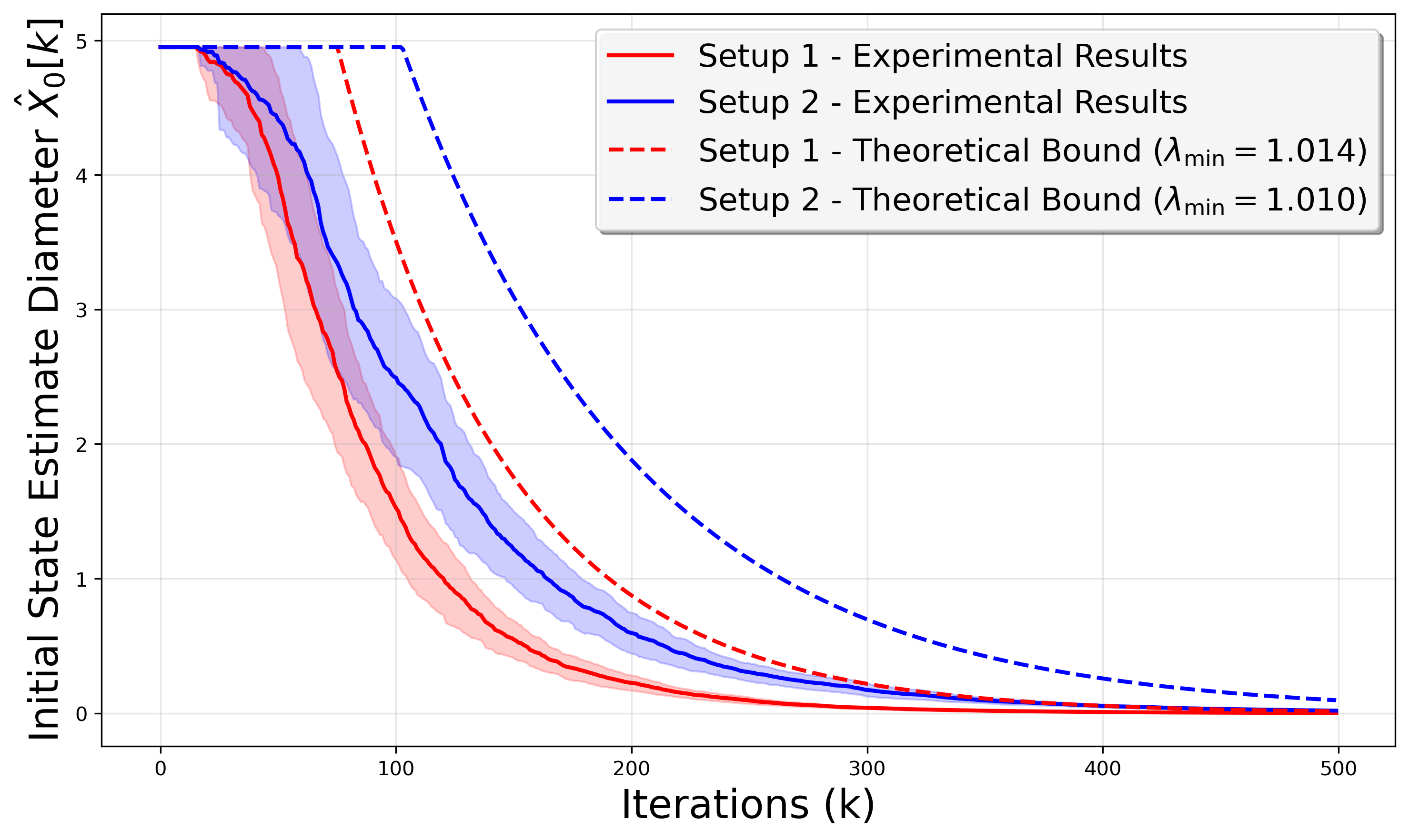}
  \Description{Line plot showing the diameter of the initial-state estimate decreasing over iterations, with a red empirical curve and a blue theoretical bound.}
  \caption{Evolution of $\operatorname{diam}(\hat{X}_0[k])$: 
the dark red and dark blue curves represent the mean trajectories, 
and the red and blue envelopes indicate the min--max ranges 
across 40 trials for the two systems. 
The dashed lines denote the corresponding theoretical bounds.}
  \label{fig:estimation}
\end{figure}

\begin{figure}[t]
    \centering
    \includegraphics[width=0.9\linewidth]{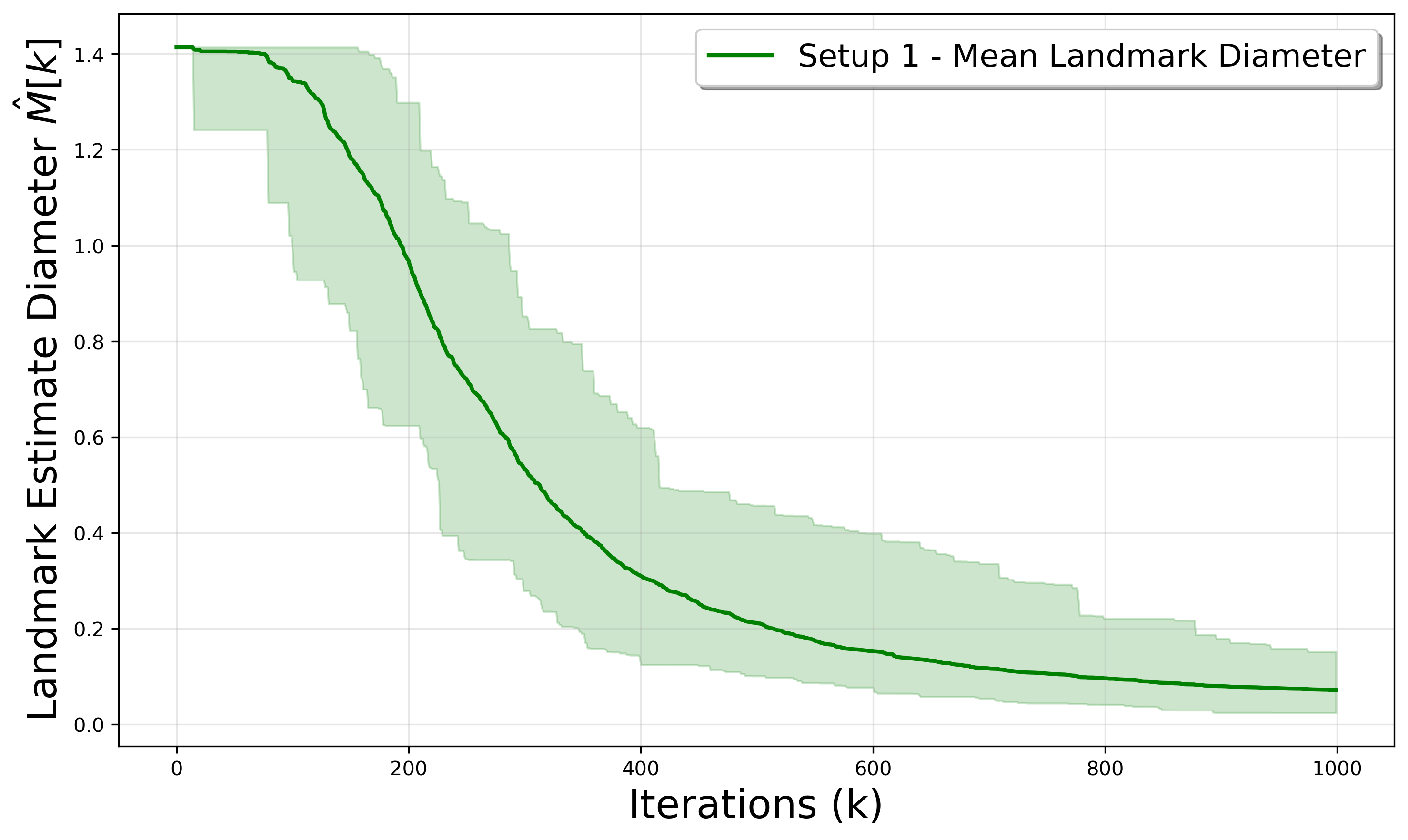}
    \caption{Evolution of $\operatorname{diam}(\hat{M}[k])$: 
    the dark green curve shows the mean trajectory, 
    and the green envelope indicates the min--max range 
    across 40 trials for Setup~1.}
    \label{fig:landmark_error}
\end{figure}
Figure~\ref{fig:estimation} shows that the diameter of the initial-state uncertainty decreases rapidly and monotonically, staying below the theoretical bound, while the spread between the minimum and maximum bounds remains small across trials. The first setup exhibits a sharper decay than Setup~2, consistent with its larger $\lambda_{\min}$. Moreover, the theoretical bounds in Figure~\ref{fig:estimation} appear slightly coarse relative to the empirical results. This stems from the lower bound in \eqref{seq:bound}, which assumes a worst case where recovery occurs only at the $N\bar{n}$-th step. In practice, recovery happens earlier, allowing more frequent executions of \hyperref[alg:EST]{\textsc{Estimate}} and thus a faster reduction in uncertainty.

\paragraph{Estimating the Landmark.}
We have not derived theoretical results for the convergence rate of $\hat{M}[k]$, 
but from~\eqref{recur} we know that it is nonexpanding. 
However, $\hat{M}[k]$ may not converge to a singleton. 
From Theorem~\ref{thm:exp_decay}, we know that the initial state estimate $\hat{X}_0[k]$ 
shrinks exponentially to a singleton. 
Thus, referring to \hyperref[EST:line9]{Line 9} of \hyperref[alg:EST]{\textsc{Estimate}}, for large $k$, the reachable sets $\textsc{Reach}\!\left(A, B,u_{0:j-1}, X_0'\right)$ 
for $j \ll k$ are very small in size, and the set 
$\textsc{Reach}\!\left(A, B, u_{0:j-1}, X_0'\right) \oplus \mathcal{B}(0, r)$ 
is approximately a ball of radius $(r+\varepsilon)$, for some small $\varepsilon>0$. 
If an oracle control law continuously stabilizes $x_k$ at a fixed location within the sensing region, so that the measurements are persistently available, then $\hat{M}[k]$ would stabilize around the intersection 
$\mathcal{B}(x_k, r+\varepsilon) \cap M$. 
This illustrates that without sufficient exploration, $\hat{M}[k]$ can stagnate around a bounded region rather than collapsing to a point. 
However, if the trajectory generated by arbitrary controls sufficiently explores the sensing region, these successive intersections are expected to progressively contract the landmark estimate. 
We present empirical results from 40 experiments conducted using Setup~1 shown in Figure~\ref{fig:landmark_error} over 1000 iterations. 
Empirically, $\hat{M}[k]$ continues to shrink as $k$ increases, but at a slower rate than $\hat{X}_0[k]$.

\section{Conclusion}

In this work, we presented a theoretical framework for active localization of unstable linear systems under coarse, single-bit sensing. Our analysis established the fundamental limits of localization in terms of instability and landmark uncertainty. Contrary to intuition, we have shown that instability can actually facilitate localization. We also provided an algorithm guaranteed to estimate the initial state with an exponential convergence rate and validated our results through experiments using different setups. Future work includes extending the analysis to account for model mismatch and system disturbances, determining the limits of localization in more general systems such as nonlinear and hybrid systems, and exploring more complex sensor models---such as keyframes and segmentations.

\bibliographystyle{ACM-Reference-Format}
\bibliography{software,sample-base}

\appendix

\section{Proofs of Propositions and Corollaries}
\subsection{Proof of Proposition~\ref{prop:SVP_exist}}
\label{appendix:SVP_existproof}

\begin{proof}
For each $p \in \mathbb{S}^{n-1}$, define the open cap
\[
C(p,\alpha)=\{\,x\in \mathbb{S}^{n-1} : \langle x,p\rangle > \alpha \,\}.
\]
For every \(x\in S^{n-1}\), $\alpha \in [0,1)$ we have \(x\in C(x,\alpha)\). Thus the family \(\{C(p,\alpha)\}_{p\in S^{n-1}}\) covers the compact n-dimensional sphere, meaning there exists a finite sub-cover
\[
\mathbb{S}^{n-1} = \bigcup_{i=1}^N C(p_i,\alpha).
\]
Any of these finitely many sub-covers with the minimum number of basis vectors form an $\alpha$-SVP.
\end{proof}

\subsection{Proof of Proposition~\ref{prop:SVP_bound}}
\label{appendix:SVP_boundproof}

\begin{proof}
Without loss of generality we work on $\alpha$-SVP of $\scalebox{1.2}{$\partial$}\mathcal{B}(0,1)$. (apply \(x\mapsto x-c\) if needed). First, we show \(\langle w,p_i\rangle\ge\alpha \ \forall w \in R_i\). Suppose for contradiction there is a \(w\in R_i\) with \(\langle w,p_i\rangle<\alpha\).  Coverage of \(\mathbb{S}^{n-1}\) by the caps \(C(p_j,\alpha)=\{w:\langle p_j,w\rangle\ > \alpha\}\) forces \(w\in C(p_j,\alpha)\) at least for one \(j\neq i\), so \(\langle p_j,w\rangle\ >\alpha>\langle p_i,w\rangle\), contradicting \(w\in R_i\) by \eqref{eq:svp_def}.  Hence \(R_i\subset C(p_i,\alpha)\) and \(\langle w,p_i\rangle\ge\alpha\) for all \(w\in R_i\). Now let \(w_1,w_2\in R_i\) and decompose these vectors in two orthogonal directions:
\[
w_k = \langle w_k,p_i\rangle\,p_i + z_k,\text{where} \ z_k\perp p_i \ \text{for} \ k=1,2
\]
$w_k \in R_i$ implies $\|w_k\| = 1$, hence $\|z_k\|=\sqrt{1-\langle w_k,p_i\rangle^2}$. Then by Cauchy-Schwarz:
\[
\langle w_1,w_2\rangle
= \langle w_1,p_i\rangle\,\langle w_2,p_i\rangle + \langle z_1,z_2\rangle
\;\ge\;
a_1a_2 - \|z_1\|\|z_2\|,
\]
with \(a_k=\langle w_k,p_i\rangle\ge\alpha\).  A quick check shows that the minimum occurs at the boundary, i.e., $a_1 = a_2 = \alpha$, yielding
\[
\langle w_1,w_2\rangle \;\ge\; 2\alpha^2 - 1.
\]
\end{proof}

\subsection{Proof of Proposition~\ref{prop:SVP_ball}}
\label{appendix:SVP_ballproof}

\begin{proof}
Let \(x\in \mathcal{B}(m,r)\) and set
\[
v \;=\; \frac{m - x}{r}.
\]
If \(v=0\), then \(m=x\) and \(int(R_i)\subseteq \mathcal{B}(m,r)\) for all i holds immediately.  Otherwise \(0<\|v\|\le1\), and we choose \(i\) so that the normalized $v$, which we call $\hat{v}$ lies in \(\tilde{R}_i\) where \(\tilde{R}_i\) is one of the normalized partitions. Such an $i$ can always be chosen because \(\bigcup_i\tilde{R}_i=\mathbb{S}^{n-1}\). Now let \(y\in R_i\) be arbitrary, and define
\[
z \;=\; \frac{y - x}{r}.
\]
Then \(\|z\|=1\) and \(z\in \tilde{R}_i\).  By Proposition ~\ref{prop:SVP_bound}, $\langle z,\,\hat v\rangle \geq 2\alpha^2-1 = 2 (\frac{\sqrt{3}}{2})^2-1=  \frac{1}{2}$. Hence
\[
\langle z,\,v\rangle
= \|v\|\;\langle z,\hat v\rangle
\;\ge\;\tfrac12\,\|v\|
\;\ge\;\tfrac12\,\|v\|^2.
\]
as $\|v\|\leq1$. Since \(\|z\|=1\), a direct calculation using the above inequality gives
$
\|z - v\|^2
= \|z\|^2 + \|v\|^2 - 2\langle z,v\rangle
\;\le\;1,
$
so \(\|z - v\|\le1\).  Therefore
$
\|y - m\|
= \|\,r z - r v\|
= r\,\|z - v\|
\;\le\;r,
$ thus \(y\in \mathcal{B}(m,r)\). Because \(y\) was arbitrary, \(R_i\subseteq \mathcal{B}(m,r)\). Since we are given $x \in \mathcal{B}(m,r)$ and the ball $\mathcal{B}(m,r)$ is convex, we conclude that $int(R_i)\subseteq \mathcal{B}(m,r)$.
\end{proof}
\subsection{Proof of Corollary~\ref{cor:recoverability_region}}
\label{appendix:recoverability_regionproof}

\begin{proof}
Let \(c_0\) be the center of the smallest enclosing ball of \( X_0\).  By Jung’s theorem there exists
\[
r_J = \sqrt{\frac{n}{2(n+1)}}\,\text{diam}( X_0)
\]
such that \( X_0\subset \mathcal{B}(c_0,r_J)\). By \eqref{cor_cond}, we obtain
\[
r_J
= \sqrt{\frac{n}{2(n+1)}}\text{diam}( X_0)
\le \frac{r(1-\eta)}{D(3-\eta)},
\]
and the condition \eqref{eq:diam_assump} in Lemma \ref{lem:recoverability_base} hold with \(r_0=r_J\).  Since \(X_0 \subseteq \mathcal{B}(c_0,r_J)\), the control \eqref{eq:control_def} ensures that for every
\(x_0\in X_0\cap \mathcal{B}(m,r)\) there is some \(i\) with \(x_{i+1}\in \mathcal{B}(m,r)\).
\end{proof}

\subsection{Proof of Corollary~\ref{cor:recoverability_region_control}}
\label{appendix:recoverability_region_controlproof}

\begin{proof}
Unrolling \(x_{k+1}=A\,x_{k}+B\,u_k\) from \(k=0\) to \(k=\bar{n}i+1\) gives
\begin{equation}\label{eq:unroll_AB}
x_{\bar{n}i+1}
= A^{\,\bar{n}i+1}x_{0}
+ \sum_{k=0}^{\bar{n}i}A^{\,\bar{n}i-k}B\,u_k.
\end{equation}
Split the sum at \(k=(i-1)\bar{n}\) and, by setting \(j = k - (i-1)\bar{n}\) in the second sum, we obtain:
\begin{equation} \label{changevar}
\sum_{k=0}^{\bar{n}i}A^{\,\bar{n}i-k}B\,u_k
=\sum_{k=0}^{(i-1)\bar{n}}A^{\,\bar{n}i-k}B\,u_k
+\sum_{j=1}^{\bar{n}}A^{\,\bar{n}-j}B\,u_{(i-1)\bar{n}+j}.
\end{equation}
Substitute \eqref{eq:control_def_AB} into \eqref{changevar}. By a change of variables from $\bar{n}-j$ to $j$ for the second summation on RHS of \eqref{changevar}, the sum becomes
\begin{equation}\label{eq:sum_expand}
\begin{aligned}
&\sum_{j=1}^{\bar{n}} A^{\bar{n}-j}B\,u_{(i-1)\bar{n}+j} =
\sum_{j=1}^{\bar{n}} A^{\bar{n}-j}B\,(A^{\bar{n}-j}B)^{\top} W^{-1}\Bigl[(r-\\
&\;\|A^{\bar{n}i+1}-I\|R_0)p_i \; - (A^{\bar{n}i+1}-I)c_0 - \sum_{k=0}^{(i-1)\bar{n}} A^{\bar{n}i-k}B\,u_k \Bigr].
\end{aligned}
\end{equation}
With the change of variables \(k=\bar{n}-j\), the outer sum becomes
\[
\sum_{j=1}^{\bar{n}} 
   A^{\bar{n}-j}B\,(A^{\bar{n}-j}B)^{\top}
= \sum_{k=0}^{\bar{n}-1} A^{k}B B^{\top}(A^{k})^{\top}
= W.
\]
Therefore $W$ cancels with $W^{-1}$ and \eqref{eq:sum_expand} reduces to
\[
(r-\|A^{\bar{n}i+1}-I\|R_0)p_i 
- (A^{\bar{n}i+1}-I)c_0 
- \sum_{k=0}^{(i-1)\bar{n}} A^{\bar{n}i-k}B\,u_k.
\]
  Thus substituting the above into \eqref{changevar}, then \eqref{changevar} into \eqref{eq:unroll_AB} gives
\[
x_{\bar{n}i+1}
= A^{\,\bar{n}i+1}x_{0}
+ \bigl(r - \|A^{\,\bar{n}i+1}-I\|\,R_{0}\bigr)\,p_{i}
- \bigl(A^{\,\bar{n}i+1}-I\bigr)\,c_{0}.
\]
This coincides with \eqref{eq:grouped} up to a scaling of the time-indices by multiples of \(\bar{n}\).  Moreover, replacing \(D\) by \(\bar D\) sets \eqref{general_cond} as a sufficient condition to repeat the same reasoning between \eqref{eq:grouped}–\eqref{eq:in_cell}. Combining the result from Corollary~\ref{cor:recoverability_region} as well, we complete the proof.
\end{proof}

\subsection{Proof of Proposition~\ref{prop:matrixnorm}}
\label{appendix:matrixnormproof}

\begin{proof}
The proof expresses $A$ in its Jordan form and analyzes a single Jordan block, 
as the argument extends by taking maximum.
By binomial expansion, we show powers' of Jordan blocks also decompose into a diagonal and nilpotent part.
Then we expand inverse of their difference from identity matrix via the Neumann series, showing that $
\|(A^k - I)^{-1}\|_2$ decays as a polynomial in $k$ times $(\lambda_{\min}^k - 1)^{-1}$.
This product is then bounded by a slower exponential 
$\lambda_{\min}^{-c k}$ for some $0 < c < 1$. Consider the Jordan form of \(A\) and set
\[
A = V J V^{-1}, \quad J = \bigoplus_i J_i, \quad J_i = \lambda_i I + N \ \forall i, \quad N^{d_i} = 0
\]

where \(\bigoplus_i J_i\) denotes the direct sum of Jordan blocks \(J_i\), each of block size \(d_i \times d_i\). Since $J$ is a direct sum of Jordan blocks, $J_i$, it is sufficient to consider the behavior of powers of a single Jordan block. We can show by binomial expansion that for each \(i\), \(J^k_i\) is also composed of a diagonal plus nilpotent part,
\begin{align*}
J_i^k &= (\lambda_i I + N)^k 
       = \lambda_i^k I + \sum_{j=1}^{d_i-1} \binom{k}{j} \lambda_i^{\,k-j} N^j \\
      &:= \lambda_i^k I + N' 
       \implies J_i^k - I = (\lambda_i^k - 1)I + N'.
\end{align*}
where \((N')^{d_i} = 0\) since each term in \(N'\) contains at least one factor of \(N\). Thus, the Neumann series expansion gives
\[
(J_i^k - I)^{-1} = \sum_{j=0}^{d_i-1} \frac{(-1)^j}{(\lambda_i^k - 1)^{j+1}} (N')^j.
\]
Considering the summation form of $N'$, each term includes a binomial factor 
$\binom{k}{j} = O(k^j)$ for $j \le d_i - 1$, leading to at most polynomial growth in $k$. 
Since $|\lambda_i|^{\,k-1} > |\lambda_i|^{\,k-j}$ for all $j \ge 1$, we may uniformly bound 
the exponential part by $|\lambda_i|^{\,k-1}$. Hence, $
\|N'\|_2 \le C_i\, k^{\,d_i-1} |\lambda_i|^{\,k-1}$, where $C_i > 0$ depends only on $d_i$ and $\|N\|$. By submultiplicativity this gives $\|(N')^j\|_2 \le C\,k^{\,j(d_i-1)}\,|\lambda_i|^{\,kj-j}$. Therefore,
\[
\label{eq:neumann-step}
\|(J_i^k - I)^{-1}\|_2
\le \sum_{j=0}^{d_i-1} \frac{\|(N')^j\|_2}{\gamma^{\,j+1}} 
\le C\,k^{{d_i}^2}\,\gamma^{-1}\sum_{j=0}^{d_i-1}\Bigl(\tfrac{|\lambda_i|^{\,k-1}}{\gamma}\Bigr)^{j} 
\]
where $\gamma = |\lambda_i^k - 1|$. We now upper-bound the term inside the summation at RHS of the second inequality above:  
\[
\begin{aligned}
\frac{|\lambda_i|^{k-1}}{\gamma} \le \frac{1}{|\lambda_i| - |\lambda_i|^{-(k-1)}}
\le {(\lambda_{\min} - 1)}^{-1}.
\end{aligned}
\]
where the first inequality comes from cancelling $|\lambda_i|^{k-1}$ terms plus reverse triangle inequality, and the last from $|\lambda_i|  \ge \lambda_{\min}>1$. Thus,
\[
\|(J_i^k - I)^{-1}\|_2
\le \frac{C\,k^{{d_i}^2}D_i}{\lambda_{\min}^k-1},
\ \text{where} \ D_i := \sum_{m=0}^{d_i-1}(\lambda_{\min}-1)^{-m}. \label{eq1}
\]
Then $\bigl\lVert(A^k - I)^{-1}\bigr\rVert_2 
\;\le\;
\kappa\max_i\bigl\lVert(J_i^k - I)^{-1}\bigr\rVert_2$ with \(\kappa = \|V\|_2\|V^{-1}\|_2\). Combining this with the previous equation, we get
\[
\bigl\lVert(A^k - I)^{-1}\bigr\rVert_2
\;\le\;
\frac{\kappa\,D\,C\,k^d}{\lambda_{\min}^k - 1}
\;\le\;
K\,\lambda_{\min}^{-c k},
\quad \forall k >0.
\]
where $D = \max_i D_i$ and $d = \max_i d_i^2$. The last inequality holds for some constant $K>0$ after fixing $c\in(0,1)$, since the preceding expression can be uniformly bounded over all $k>0$.
\end{proof}

\end{document}